\documentclass{article}

\usepackage[preprint]{corl_2026} 

\usepackage{amsmath}
\usepackage{mathtools}
\usepackage{amssymb}
\usepackage{amsfonts}   
\usepackage{amsthm}     
\usepackage{dsfont}

\usepackage[pdftex]{graphicx}
\usepackage{wrapfig}
\usepackage{caption}
\usepackage{subcaption}
\usepackage{algorithm}
\usepackage{algpseudocode}
\usepackage{bm}

\usepackage{booktabs}   
\usepackage{makecell}

\usepackage[table]{xcolor}  
\definecolor{mydarkblue}{rgb}{0,0.08,0.45}

\definecolor{base}{HTML}{AF56BF}
\definecolor{ref}{HTML}{F58A45}
\definecolor{task}{HTML}{407CB6}
\definecolor{pps}{HTML}{55A86E}

\definecolor{revcolor}{HTML}{B8860B}
\newcommand{\hl}[1]{\textcolor{revcolor}{#1}}
\renewcommand{\hl}[1]{#1}

\usepackage{hyperref}
\hypersetup{
    colorlinks=true,
    linkcolor=mydarkblue,
    citecolor=mydarkblue,
    filecolor=mydarkblue,
    urlcolor=mydarkblue,
}

\usepackage{animate}

\usepackage[]{mdframed}

\usepackage{xparse}     
\usepackage{expl3}      
\usepackage{xspace}
\usepackage{lipsum}     
\usepackage{algorithm, algpseudocode}
\usepackage{microtype}
\usepackage{cleveref}
\newtheorem{proposition}{Proposition}

\crefname{section}{Sec.}{Secs.}
\crefname{figure}{Fig.}{Figs.}
\crefname{table}{Tab.}{Tabs.}
\crefname{equation}{Eq.}{Eqs.}

\DeclarePairedDelimiterX{\infdivx}[2]{(}{)}{%
  #1\;\delimsize\|\;#2%
}

\ExplSyntaxOn
\newcommand\latinabbrev[1]{
  \peek_meaning:NTF . {%
    \textit{#1}\@}%
  { \peek_catcode:NTF a {%
      \textit{#1}.\@ }%
    {\textit{#1}.\@}}}
\ExplSyntaxOff

\def \MethodName {Proxy Policy Steering\xspace}
\def \MethodAcronym {PPS\xspace}
\def \NickName {PPS\xspace}

\title{\MethodName}

\definecolor{linkpink}{RGB}{237,0,140}

\author{
    \textbf{
    Chuanruo Ning$^{*}$
    \quad
    Tianrui Wang$^{*}$
    \quad
    Wei-Chiu Ma$^{\dagger}$
    \quad
    Kuan Fang$^{\dagger}$
    } \\[1mm]
    Cornell University \\[1mm]
    {\href{https://ppsteering.github.io}
      {\textcolor{linkpink}{\texttt{https://ppsteering.github.io}}}
    }
}

\begin{document}
\maketitle
\vspace{-6mm}
\let\thefootnote\relax\footnotetext{$^{*}$Equal contribution. $^{\dagger}$Equal advising.}


\begin{abstract}
\hl{Generalist robot policies} carry broad manipulation priors from large-scale data, but specializing them to a new task remains the deployment bottleneck. This requires \hl{eliciting} task-specific behavior from limited demonstrations without \hl{degrading} their broad capabilities. We introduce \MethodName{} (\MethodAcronym), \hl{an inference-time adaptation method that resolves this challenge by training two lightweight proxy policies whose calibrated velocity-space difference steers the frozen base sampler.} A \emph{reference proxy} models the frozen base's behavior on target-task observations, and a \emph{task proxy}, initialized from the reference, captures how this behavior changes under task supervision. Their difference forms a calibrated velocity-space residual that steers the frozen base sampler at every denoising step. We identify the conditions under which this residual isolates the change induced by task supervision, and validate them empirically. \hl{Because the base is never directly modified, its broad capabilities remain available at inference, including behaviors such as recovery from failure that the demonstrations themselves do not exercise. Adaptation requires only forward velocity predictions from the base, making \MethodAcronym{} lightweight to train and applicable even without access to the base's parameters.} On 8 real-world and 4 simulation manipulation tasks, \MethodAcronym{} lifts the state-of-the-art $\pi_{0.5}$ base policy by 55\% absolute success rate on average, with \hl{zero-to-one} gains on tasks the base never solves, while preserving the base's broad capabilities. \MethodAcronym{} outperforms LoRA fine-tuning, from-scratch specialists, residual policies, and prior inference-time steering methods.
\end{abstract}


\section{Introduction}

\begin{figure*}[t]
    \centering
    %
    %
    %
    \includegraphics[width=\linewidth]{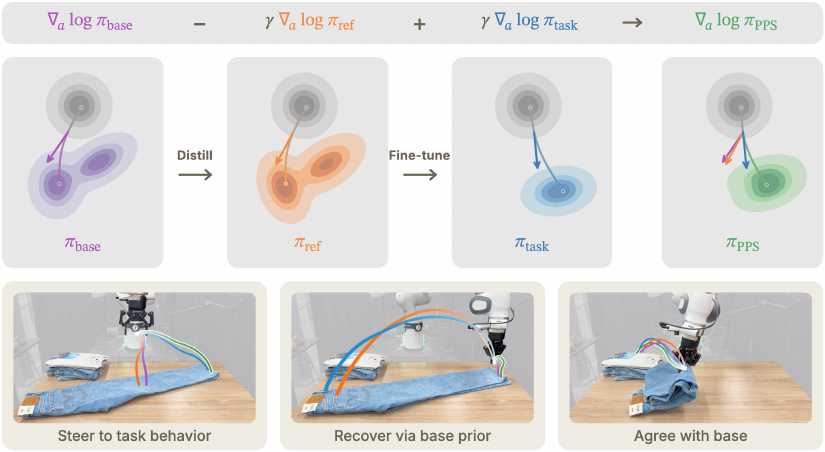}
    \vspace{-3mm}
    \caption{\hl{\textbf{\MethodName (\NickName).} We distill the frozen \textcolor{base}{base policy} $\pi_{\mathrm{base}}$ into a lightweight \textcolor{ref}{reference proxy} $\pi_{\mathrm{ref}}$ on target-task observations, then fine-tune $\pi_{\mathrm{ref}}$ into a \textcolor{task}{task proxy} $\pi_{\mathrm{task}}$ on demonstrations. Their \textit{velocity-space residual} steers the frozen base at every denoising step, yielding the \textcolor{pps}{PPS} action that injects task-specific behavior while preserving the base's general priors. 
    The bottom panels show three representative cases: a state where the base needs task guidance (residual steers toward the task behavior), a state not covered by task demonstrations (residual cancels, base prior recovers), and a state where the proxies agree with the base (residual vanishes).
    }
    }
    \vspace{-3mm}
    \label{fig:teaser}
\end{figure*}

\hl{Generalist robot policies trained on large-scale data, particularly Vision-Language-Action (VLA) models~\cite{brohan2022rt, black2024pi_0, team2025gemini, ye2026world}, bring broad manipulation competence into a single policy that generalizes across diverse embodiments, scenes, and instructions.} Yet specializing them to a particular downstream task remains a bottleneck in real-world deployment~\cite{intelligence2025pi, zhao2025mos}, since the deployment distribution rarely matches the pretraining data, and task adaptation typically operates on only a handful of demonstrations. This presents a fundamental tension between acquiring task-specific behavior from limited data and preserving the broad capabilities these generalist policies were trained to provide.

Existing adaptation methods either modify the base model's parameters or steer its sampling process, and both paths have notable limitations. Modifying the base through \emph{fine-tuning}~\cite{kim2025fine, driess2025knowledge} is effective when demonstrations are abundant but compromises the base's broad capabilities in the limited-demonstration regime~\cite{dey2025revla, yadav2025robust, shuttleworth2026lora}, \hl{degrading} behaviors such as recovery from failure that the demonstrations do not exercise. Fine-tuning is also computationally and memory intensive at VLA scale, even with low-rank adapters~\cite{hu2022lora}, and requires access to model weights, ruling out closed-source or API-only base policies.
\emph{Steering}-based methods~\cite{wang2025inference, wu2025foresight, wang2025latent} keep the base intact and adjust the action sampling process instead, whether by reshaping the noise distribution that drives the sampler~\cite{wagenmaker2025steering} or by guiding the denoising trajectory with auxiliary models~\cite{du2025dynaguide}. \hl{But each reweights behaviors the base already produces on its own, and cannot reach task-specific behaviors that the base assigns negligible probability to.} How to steer a frozen base \hl{toward behaviors outside what it samples on its own} while preserving the breadth of priors it carries remains an open question.

We introduce \MethodName~(\MethodAcronym), an inference-time adaptation method for frozen flow-matching VLAs \hl{that resolves the open question above by combining the strengths of both paradigms}. Our key insight is that effective adaptation should neither steer only within the base's existing distribution nor replace the base with a narrowly trained specialist. Instead, it should \emph{estimate the direction in which task supervision would shift the base's action distribution}, and apply that direction to the base sampler without modifying any weights, \hl{so that the base's broad priors remain intact}.

To estimate this direction in practice, \MethodAcronym{} trains two lightweight \emph{proxy policies} and combines them in velocity space, as illustrated in~\Cref{fig:teaser}. A \emph{reference proxy} is distilled from the frozen base on target-task observations, capturing the base's behavior on the states the sampler will visit at inference. A \emph{task proxy} is initialized from the reference and fine-tuned on demonstrations, capturing how the base's behavior would change under task supervision. \hl{How the proxies are constructed determines whether their difference is meaningful. A small proxy trained on limited demonstrations does not faithfully reproduce the base, and its approximation error is comparable in magnitude to the task-induced change we want to extract, so any residual built from a single proxy, or from a proxy and the base, retains this error. Because the reference is distilled on the same states and the task proxy is initialized from it, the two proxies share the error, and their difference isolates the change induced by task supervision. At inference, the difference between the two proxies forms a velocity-space residual that we add to the base velocity field at every denoising step. When demonstrations reveal new task-relevant actions, the residual steers the base toward the task proxy. When the proxies agree, it cancels and the frozen base retains control, preserving broadly useful priors such as failure recovery.}

We evaluate \MethodAcronym{} across 8 real-world and 4 simulation manipulation tasks spanning both in-distribution and out-of-distribution settings. \MethodAcronym{} lifts the average success rate of the state-of-the-art $\pi_{0.5}$ base policy by 55\% absolute, with \hl{zero-to-one gains on 4 real-world tasks the base never solves}. The same procedure applied to $\pi_0$ raises its real-world success rate from 5\% to 55\%, suggesting that the velocity-space residual is a property of flow-matching steering rather than a $\pi_{0.5}$-specific trick. \MethodAcronym{} also outperforms LoRA fine-tuning, from-scratch specialists, residual policies, and existing inference-time steering methods, while preserving the base's recovery behavior throughout deployment.


\section{Related Work}
\label{sec:related}

\hl{Pretrained generalist robot policies are typically adapted to downstream tasks by fine-tuning the base weights, through full fine-tuning, LoRA~\citep{hu2022lora}, parameter merging~\citep{yadav2025robust}, or RL fine-tuning~\citep{zhang2025reinflow,pfrommer2025reinforcement,chen2025conrft,peng2026ignore}. Fine-tuning reshapes the base's action distribution toward the demonstrations, which is expensive at VLA scale and overwrites manipulation priors the demonstrations do not exercise~\citep{dey2025revla,shuttleworth2026lora}, including robust grasping, contact-rich interaction, and recovery from failure.}

\hl{A second family leaves the base unmodified and intervenes during sampling, at the input noise, the sampled output, or the denoising trajectory. DSRL~\citep{wagenmaker2025steering} optimizes the initial noise with RL, which requires non-zero base success or a large offline dataset and gains little on tasks the base can hardly solve. Residual policy learning~\citep{yuan2025policy,liu2026self,julg2025refined,xiao2025self} and candidate selection~\citep{nakamoto2024steering,wu2025foresight,li2026towards} act on the sampled output, after the base has committed to an action mode. RFS~\citep{su2026rfs} combines noise optimization with a post-hoc residual. DynaGuide~\citep{du2025dynaguide} guides the trajectory with a learned dynamics model, assuming in-distribution dynamics. \MethodAcronym{} instead applies a density-ratio correction in velocity space at every denoising step, tilting the base toward task-relevant modes before the sampler commits.}

\hl{\MethodAcronym{} adapts the proxy tuning principle from language models, where small expert and anti-expert models steer a frozen base at decoding time~\citep{liu2021dexperts,liu2024tuning,li2023contrastive,mitchell2023emulator}. Proxy tuning combines categorical logits, and its direct continuous-control analog corrects only the final clean action. We show that under the linear flow-matching path, the density-ratio composition takes an additive velocity-space form at every noise level, which lets \MethodAcronym{} reshape the mode the sampler enters. This places \MethodAcronym{} within score-space steering for diffusion and energy-based samplers. Classifier-free guidance~\citep{ho2022classifier,karras2024guiding} amplifies a conditioning signal learned during pretraining by residualizing one model's conditional and unconditional outputs. Compositional methods~\citep{du2020compositional,liu2022compositional,wang2024poco} add the scores of independently trained models, where naive composition can sample from the wrong distribution~\citep{du2023reduce}. \MethodAcronym{} inherits the additive form from this lineage. What it adds is the construction of the guidance pair: a reference proxy distilled from the base on the states the sampler will visit, and a task proxy fine-tuned from it, so that the residual isolates the change induced by task supervision rather than an arbitrary difference between two networks.}

\section{Preliminaries}
\hl{Generalist robot policies are trained on large-scale, diverse data to produce actions across many tasks, scenes, and embodiments. We focus on Vision-Language-Action (VLA) models, a family of such policies that map observations, proprioceptive states, and language instructions to robot actions.}
The conditional action distribution induced by these policies is typically denoted as $\pi_{\mathrm{base}}(a \mid o,l)$, where $o$ denotes the observation, including states, and $l$ is the language instruction. 

Rather than directly predicting a single action, state-of-the-art VLAs often use flow-matching action heads~\cite{black2024pi_0, intelligence2025pi, intelligence2026pi}, which generate actions by transforming Gaussian noise into clean actions through a learned velocity field. 
Concretely, the action head predicts a velocity field
$v_{\mathrm{base}}(x_k,k,o,l)$ over intermediate action states $x_k$ along the flow path:
\begin{equation}
    x_k = k\epsilon + (1-k)a, \quad \epsilon \sim \mathcal{N}(0,I), \quad k \in [0,1],
\end{equation}
where $k=1$ corresponds to noise and $k=0$ to a clean action.
At inference time, the model starts from $x_1 \sim \mathcal{N}(0,I)$ and integrates the predicted velocity field to obtain the final action.


Given a small target-task demonstration set $\mathcal{D}_{\mathcal{T}} = \{(o, a, l)\}$, our goal is to adapt the base VLA's action distribution on this task while preserving its broad pretraining priors, without modifying base parameters. \hl{We treat the base VLA as a query-only model whose parameters and gradients are unavailable, providing only forward velocity predictions.}

\section{\MethodName{}}
\label{ref:method}

We present \MethodAcronym{}, an adaptation method that steers a \hl{frozen generalist robot policy} toward task-successful actions without modifying its weights\hl{, developed here for VLAs with flow-matching action heads}. \Cref{sec:distribution} formulates the adaptation as a relative density correction to the base policy. \Cref{sec:method_flow-matching_proxy_guidance} shows that this correction admits an equivalent velocity-space implementation for flow-matching VLAs, which is what the inference-time interface allows. \Cref{sec:learning} describes the two-stage proxy training that makes the velocity residual a meaningful task-specific signal.

\subsection{Adapting Generalist Policies by Proxy}
\label{sec:distribution}
We frame adaptation as a relative density correction to the base policy. Inspired by product-of-experts (PoE) models~\citep{hinton2002training}, we define the adapted policy as a product between the base and a relative task correction:
\begin{equation}
    \tilde{\pi}(a \mid o,l) \propto
    \pi_{\mathrm{base}}(a \mid o,l)
    \left[
    \frac{\pi_{\mathrm{task}}(a \mid o,l)}
         {\pi_{\mathrm{ref}}(a \mid o,l)}
    \right]^{\gamma},
    \quad \gamma \geq 0.
    \label{eq:proxy_ratio}
\end{equation}
Here, $\pi_{\mathrm{ref}}$ represents the behavior before task supervision, $\pi_{\mathrm{task}}$ represents the behavior after task supervision, and $\gamma$ controls the strength of the correction. 

The formulation can be understood as approximating the adaptation direction that would be induced by fine-tuning. 
Specifically, if model weight access, computation, and data were unconstrained, one could directly fine-tune the full base VLA on $\mathcal{D}_{\mathcal T}$ to obtain a task-adapted policy $\pi_{\mathrm{task}}$. 
If $\pi_{\mathrm{ref}}$ exactly matched $\pi_{\mathrm{base}}$, then Eq.~\eqref{eq:proxy_ratio} would reduce to
\begin{equation}
    \tilde{\pi}(a \mid o,l) \propto
    \pi_{\mathrm{base}}(a \mid o,l)^{1-\gamma}
    \pi_{\mathrm{task}}(a \mid o,l)^{\gamma},
\end{equation}
which interpolates between the frozen base policy and the task-adapted policy. 
When $\gamma=1$, the base and reference terms cancel, and the adapted policy follows the task policy; when $\gamma=0$, the adapted policy falls back to the base policy.

In practice, however, directly fine-tuning the full VLA is often infeasible or undesirable. 
\MethodAcronym{} thus approximates this relative correction using lightweight proxy policies.
The \emph{reference proxy} $\pi_{\mathrm{ref}}$ is trained to capture the frozen base policy's behavior on target-task observations, while the \emph{task proxy} $\pi_{\mathrm{task}}$ is trained from demonstrations to capture how this behavior changes under task supervision. 
Their ratio provides a task-induced correction: actions made more likely by task supervision are amplified, while actions already explained by the base policy are preserved. 
This allows \MethodAcronym{} to inject task-specific expertise while retaining  useful priors of the frozen VLA, such as recovery from failures.

\subsection{Flow-Matching Proxy Guidance}
\label{sec:method_flow-matching_proxy_guidance}
The relative-correction view in Sec.~\ref{sec:distribution} provides a distribution-level formulation of proxy-based VLA adaptation. 
The remaining question is how to implement this correction for flow-matching VLAs, whose inference-time interface is a velocity field rather than a normalized action density.



To instantiate the relative correction in Eq.~\eqref{eq:proxy_ratio}, we need to express the task-induced change in the same space as the base sampler. 
Since the sampler is governed by a velocity field, \MethodAcronym{} represents this change as the difference between the task and reference proxy velocities. 
Intuitively, $v_{\mathrm{task}}-v_{\mathrm{ref}}$ captures how task supervision changes the generation dynamics beyond what is already explained by the base behavior. 
We therefore steer the base velocity field by adding this residual:
\begin{equation}
v_{\mathrm{\MethodAcronym{}}}(x,k,o,l)
=
v_{\mathrm{base}}(x,k,o,l)
+
\gamma
\left[
v_{\mathrm{task}}(x,k,o)
-
v_{\mathrm{ref}}(x,k,o)
\right],
\label{eq:guided_velocity}
\end{equation}
where $\gamma$ controls the steering strength. We drop the language input for the proxy policies since we assume a single-task setting throughout.
This velocity residual is not an ad hoc choice. Under a shared flow-matching schedule, applying the correction in \eqref{eq:proxy_ratio} to the noised marginals at each noise level\hl{, as in classifier-free guidance,} is \emph{equivalent} to the additive residual in \eqref{eq:guided_velocity}. We \hl{provide} the assumptions and derivation in \Cref{app:velocity-derivation}.

This guidance rule mirrors the cancellation property of the distribution-level correction. When task supervision changes the proxy velocity, the residual $v_{\mathrm{task}}-v_{\mathrm{ref}}$ steers the sampler toward the task proxy. When the task and reference proxies agree, the residual cancels out and the frozen base policy remains in control. 
At inference time, as shown in~\Cref{alg:inference}, \MethodAcronym{} samples $x_1 \sim \mathcal{N}(0,I)$ and integrates the guided velocity field $v_{\mathrm{\MethodAcronym{}}}$ from $k=1$ to $k=0$ using the same scheduler as the base policy. When $\gamma=0$, this exactly recovers the frozen base sampler; larger values increasingly trust the proxy residual and impose stronger task-specific steering.

Importantly, because the proxy residual is added in the shared velocity space, the proxies may condition on modalities the base does not, such as point clouds alongside RGB. The only requirement is that the base and proxies share the same action representation $x_k$ and flow-matching schedule.

\begin{figure*}[t]
\centering

\begin{minipage}[t]{0.49\textwidth}
\begin{algorithm}[H]
\caption{\MethodAcronym Training}
\label{alg:training}
\begin{algorithmic}[1]

\State Train \(v_{\mathrm{ref}}\) by distilling base denoising trajectories
on \((o, l) \sim \mathcal{D}_{\mathcal T}\) with \(\mathcal{L}_{\mathrm{ref}}\) in \eqref{eq:ref-loss}.

\State Initialize \(v_{\mathrm{task}} \leftarrow v_{\mathrm{ref}}\).

\State Train \(v_{\mathrm{task}}\) on demonstrations \((o, a) \sim \mathcal{D}_{\mathcal T}\)
with \(\mathcal{L}_{\mathrm{task}}\) in \eqref{eq:task-loss}.

\State \Return \(v_{\mathrm{ref}}, v_{\mathrm{task}}\)
\end{algorithmic}
\end{algorithm}
\end{minipage}
\hfill
\begin{minipage}[t]{0.49\textwidth}
\begin{algorithm}[H]
\caption{\MethodAcronym Inference}
\label{alg:inference}
\begin{algorithmic}[1]

\State Sample \(x_{k_1}\sim\mathcal{N}(0,I)\).

\For{\(i=1,\ldots,K-1\)}
    \State Compute \(\tilde{v}_i=v_{\mathrm{\MethodAcronym{}}}(x_{k_i},k_i,o,l)\)
    by \eqref{eq:guided_velocity}.

    \State Update \(x_{k_{i+1}}\leftarrow x_{k_i}+(k_{i+1}-k_i)\tilde{v}_i\).
\EndFor

\State \Return \(a=x_{k_K}\)
\end{algorithmic}
\end{algorithm}
\end{minipage}

\end{figure*}

\subsection{Learning Proxy Policies}
\label{sec:learning}

The guidance rule in Eq.~\eqref{eq:guided_velocity} is meaningful only if the residual $v_{\mathrm{task}}-v_{\mathrm{ref}}$ captures the change induced by task supervision, rather than arbitrary differences between two independently trained models. \MethodAcronym{} therefore trains the proxies in two stages: first, an on-policy reference distillation stage that aligns the reference proxy with the frozen base sampler; second, a task specialization stage that adapts the proxy to demonstrations, as shown in~\Cref{alg:training}.

\paragraph{On-policy reference distillation.}
We first train the reference proxy to mimic the frozen base model on the states where inference-time guidance will be applied. For each target-task observation and instruction $(o,l)$, we roll out the frozen base sampler from noise and collect intermediate action states $x_{k_i}$ along its generation trajectory. The reference proxy is then trained by velocity distillation:
\begin{equation}\label{eq:ref-loss} \mathcal{L}_{\mathrm{ref}}=\mathbb{E}_{(o,l)\sim\mathcal{D}_{\mathcal T},\,\epsilon\sim\mathcal{N}(0,I)}\sum_{i=1}^{K}\left\|v_{\mathrm{ref}}(x_{k_i},k_i,o)-v_{\mathrm{base}}(x_{k_i},k_i,o,l)\right\|^2. \end{equation}
This on-policy distillation ensures that the reference proxy matches the base behavior in the region where the proxy residual will later be applied.

\paragraph{Task specialization.}
We then initialize the task proxy from the reference proxy and train it on task demonstrations with the standard flow-matching objective:
\begin{equation}
\label{eq:task-loss} \mathcal{L}_{\mathrm{task}}=\mathbb{E}_{(o,a)\sim\mathcal{D}_{\mathcal T},\,\epsilon\sim\mathcal{N}(0,I),\,k\sim\mathrm{Beta}(1.5,1)}\left\|v_{\mathrm{task}}(x_k,k,o)-(\epsilon-a)\right\|^2,\quad x_k=k\epsilon+(1-k)a. \end{equation}
Because the task proxy is initialized from the reference proxy, the two proxies share architecture, initialization, and approximation errors before task supervision. Their difference therefore cleanly isolates the task-induced change, rather than artifacts from independently trained networks. \hl{A small proxy trained on limited demonstrations carries approximation error comparable to the task-induced change itself, and training the proxies independently, or replacing the reference with the base, leaves this error in the residual. Fine-tuning from the reference keeps the two errors close, since limited demonstrations move the network only along the directions they support, which is why proxy tuning pairs an expert with a matched anti-expert rather than with the base~\citep{liu2024tuning}.}

\section{Experiments}
To evaluate and better understand \MethodAcronym{}, we design our experiments around three questions: (1) Does \MethodAcronym{} improve VLA performance on downstream tasks? (2) Does \MethodAcronym{} preserve the general manipulation priors of VLAs while eliciting task-specific behaviors? (3) What drives the effectiveness of \MethodAcronym{}, and which design choices are necessary?


\subsection{Experimental Setup}
\noindent\textbf{Tasks.}
We evaluate \MethodAcronym on 8 real-world tasks and 4 simulation tasks, as shown in~\Cref{fig:tasks_results}. These tasks include both settings that resemble the VLA's training distribution, such as {tissue wiping} and {lever pressing}, and challenging (out-of-distribution) tasks where the VLA rarely succeeds, such as {coffee brewing}, {jeans folding}, and {flower arrangement}.
We collect 50 demonstrations for each task.
We further evaluate \MethodAcronym on cross-modality steering tasks, where the proxy policies are trained with additional modality inputs such as point clouds or audio to steer an RGB-only base policy. Please see \Cref{sec:multimodal} for cross-modality steering results.

\noindent\textbf{Baselines.}
\textbf{$\boldsymbol{\pi_{0.5}}$}~\cite{intelligence2504pi0} is a state-of-the-art VLA pretrained on diverse robotics data and fine-tuned on the DROID dataset~\cite{khazatsky2024droid}. We use it as our frozen base model. \textbf{Specialist} is a flow-matching policy trained from scratch on the task demonstrations, with the same architecture as our proxy models.
\textbf{Residual} uses the same architecture as the Specialist policy, but additionally takes the base policy's predicted action as input and learns a correction on top of it.
\textbf{DSRL}~\cite{wagenmaker2025steering} adapts the base model by training a policy over its initial noise distribution via offline or online RL. We compare with its offline Noise-Aliased variant. 
\textbf{LoRA}~\cite{hu2022lora} fine-tunes the base model with low-rank adapters, using the same trainable-parameter budget as \MethodAcronym{}. It represents the standard weight-fine-tuning approach.


\noindent\textbf{Metrics.}
We report task success rates with an action step budget of 1200. We run 10 rollouts per task per method in the real world and 100 in simulation, both with randomized initial configurations.

\noindent\textbf{Implementation details.}
Each proxy policy has 32M parameters (\emph{i.e.}, $1\%$ of the base model size), using DINOv3~\cite{simeoni2025dinov3} as the visual encoder and an action expert modified from $\pi_{0.5}$~\cite{intelligence2504pi0}. We train the reference and task proxies for 20,000 steps each, with a batch size of 64; same for baselines.

\begin{figure*}[t]
    \centering
    \includegraphics[width=\linewidth]{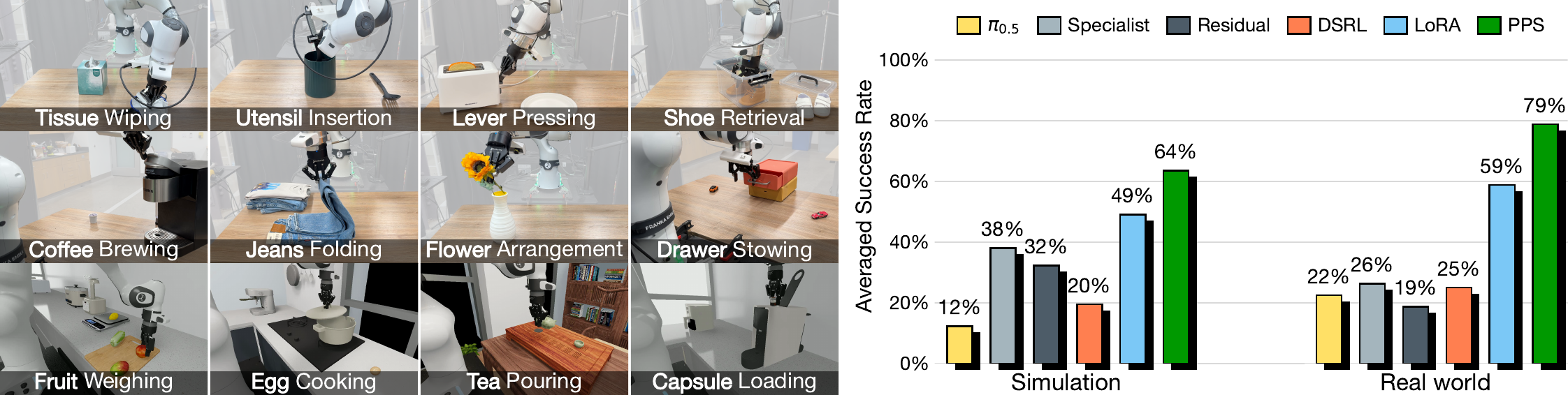}
    \caption{\textbf{Tasks and results.} \textbf{Left:} the manipulation tasks used in our evaluation: 8 real-world tasks (top two rows) and 4 simulation tasks (bottom row), covering both in-distribution (e.g., Tissue, Fruit) and out-of-distribution (e.g., Jeans, Capsule) settings for the base VLA. \textbf{Right:} success rates averaged over tasks in the simulation and real world. Per-task numbers are reported in \Cref{sec:appendix_per_task}.}
    \label{fig:tasks_results}
\end{figure*} 

\subsection{Comparative Results}
\label{sec:comparative}

As shown in \Cref{fig:tasks_results} (right), \MethodAcronym{} substantially improves the success rate of $\boldsymbol{\pi_{0.5}}$ on downstream tasks, achieving a 55\% absolute gain averaged across all 12 tasks (see \Cref{sec:appendix_per_task} for the per-task breakdown). By steering the denoising trajectory with a velocity residual learned from demonstrations, \MethodAcronym{} enables the base VLA to solve tasks that it could not complete on its own.

\hl{In comparison, \textbf{DSRL} improves over the base where it already has non-zero success (tissue wiping) but not where it has none (coffee brewing), since RL-based adaptation needs either a large offline dataset or non-zero base success to bootstrap. \MethodAcronym{} also significantly outperforms \textbf{Specialist}, a from-scratch policy trained on the same demonstrations. A lightweight model trained on 50 demonstrations cannot form a robust policy on its own, yet the same architecture and data suffice as guidance that steers the VLA toward the demonstrated modes~(\Cref{eq:proxy_ratio}) and elicits its broad priors. \textbf{Residual} learning underperforms even Specialist, likely because the residual is applied only after the base has sampled a clean action not well aligned with the task. Most importantly, \MethodAcronym{} outperforms \textbf{LoRA} under the same trainable-parameter budget. LoRA reshapes the action distribution to fit the demonstrations and overwrites pretrained priors, whereas \MethodAcronym{} injects the task signal through velocity steering and preserves them. \Cref{fig:robustness} makes this concrete: \MethodAcronym{} recovers from failures absent from the demonstrations, whereas LoRA overfits to them. Under mid-rollout disturbances, \MethodAcronym{} loses only $7.8$ points of success rate, against $15.8$ for LoRA and $26.1$ without the reference proxy (\Cref{sec:appendix_perturbation}).}

\hl{\MethodAcronym{} is also efficient and transfers across base policies. It trains $5\times$ faster than LoRA (2.9h vs.\ 16.3h on a single A6000) and adds only 6\,ms of inference latency per action chunk, since the two 32M proxies run in parallel with the frozen base (\Cref{sec:appendix_efficiency}). Applying the same procedure to $\pi_0$ raises its real-world success rate from 5\% to 55\%, with zero-to-one gains on all six tasks $\pi_0$ never solves (\Cref{sec:appendix_pi0}), indicating that the velocity-space residual is not specific to $\pi_{0.5}$.}




\begin{figure*}[t]
    \centering
    \includegraphics[width=\linewidth]{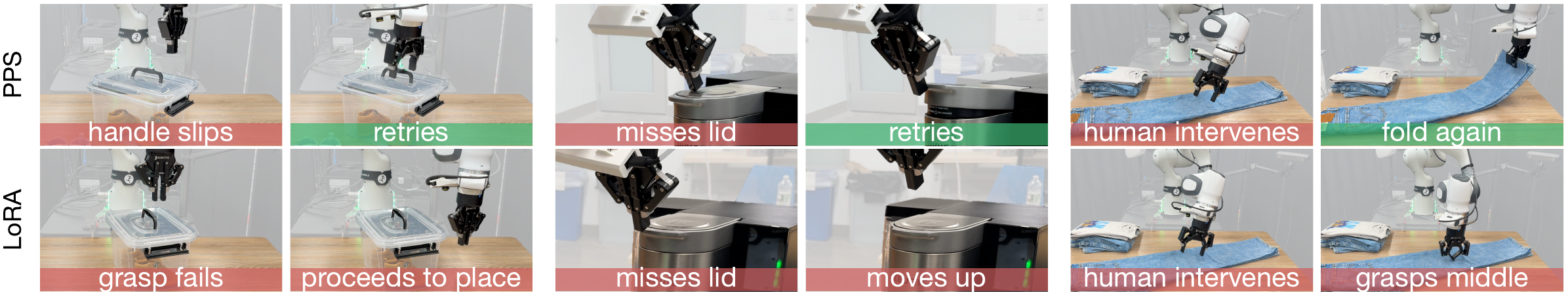}
    \vspace{-4mm}
    \caption{\textbf{Recovery behaviors.} \MethodAcronym recovers from failures even though no recovery trajectories appear in the demonstrations, whereas LoRA fails the task due to overfitting to the demonstrations. See \Cref{sec:appendix_perturbation} for quantitative results under perturbations.}
    \label{fig:robustness}
\end{figure*}



\subsection{Analysis on \MethodAcronym}
\label{sec:ablation}

We analyze \MethodAcronym along two axes: the steering strength $\gamma$ that controls how strongly the base sampler is tilted toward the proxy ratio, and the calibration of the residual itself.

\paragraph{Steering strength.} \Cref{fig:gamma} sweeps $\gamma$ on the simulated tasks. Setting $\gamma = 0$ recovers the unsteered base sampler, and larger values impose stronger task-specific steering---matching the trade-off characterized in \Cref{sec:method_flow-matching_proxy_guidance}. The optimal $\gamma$ in the 50-demonstration regime lies inside $(0.4, 0.6)$, balancing the base's general priors against the task-specific modes amplified by the proxy ratio. We fix $\gamma=0.4$ for every task in our experiments, without per-task tuning.

\paragraph{Proxy training recipe.} The velocity residual is useful only if it isolates the change induced by task supervision. Specifically, Eq.~\eqref{eq:guided_velocity} requires two conditions:
(C1) the reference proxy matches the base velocity on task-relevant observations, and (C2) the task and reference proxies share approximation errors and randomness. \Cref{fig:ablation_avg} ablates the design choices that enforce these conditions.
\textbf{w/o\ ref} replaces $v_{\mathrm{ref}}$ with the frozen base $v_{\mathrm{base}}$. In this case, C1 holds trivially, but C2 is broken. The small task proxy carries architecture-specific approximation errors~\cite{liu2024tuning}, but the base does not share those errors. The residual mixes task signal with proxy-specific bias, and the success rate drops from 64\% to 55\%. \Cref{sec:appendix_cosine} quantifies this bias: the task--reference residual tracks a genuine fine-tuning shift far more closely than task--base.
\textbf{w/o\ vel} supervises the reference proxy on the base's clean predicted actions. Thus, the reference fits the base only at the clean endpoint, not along the denoising trajectory, violating C1. The residual then carries leftover base behavior that the guided sampler cannot cancel; success rate drops to 50\%.
\textbf{w/o\ tune} trains the task proxy from scratch instead of initializing it from the reference proxy. The two proxies then differ not only in task supervision but also in random initialization and optimization trajectory, so C2 fails and the residual no longer points cleanly along task-relevant directions; the success rate drops to 48\%.
Each ablation isolates one of the three design choices and confirms that the residual encodes a genuine task-relevant direction only when all three are in place.

\paragraph{Failure analysis.} We categorize failures into four root causes: \textbf{incorrect mode} such as trying to place the toy in the drawer before opening it, \textbf{imprecise execution} error due to imprecise low-level object interactions, \textbf{out-of-distribution (OOD) states} where the state drifts to configurations unseen at training time, and \textbf{other} failures such as hardware glitches and randomness. As shown in~\Cref{fig:failure}, \hl{compared to the frozen base $\pi_{0.5}$, \MethodAcronym{} cuts incorrect-mode failures by 93\%, reflecting the task-specific direction provided by the proxy residual. Compared to Specialist, \MethodAcronym{} cuts OOD-state failures by 87\%, reflecting the base prior that keeps the sampler inside well-supported state regions.}

\begin{figure*}[t]
    \centering
    \begin{subfigure}[b]{0.288\linewidth}
        \centering
        \includegraphics[width=\linewidth]{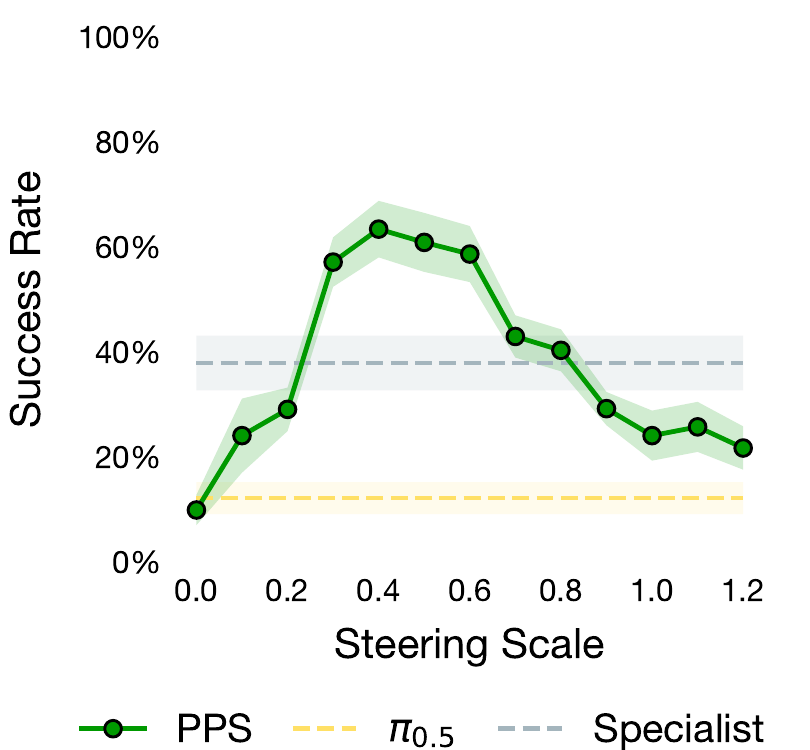}
        \caption{Sensitivity to $\gamma$}
        \label{fig:gamma}
    \end{subfigure}
    \hfill
    \begin{subfigure}[b]{0.288\linewidth}
        \centering
        \includegraphics[width=\linewidth]{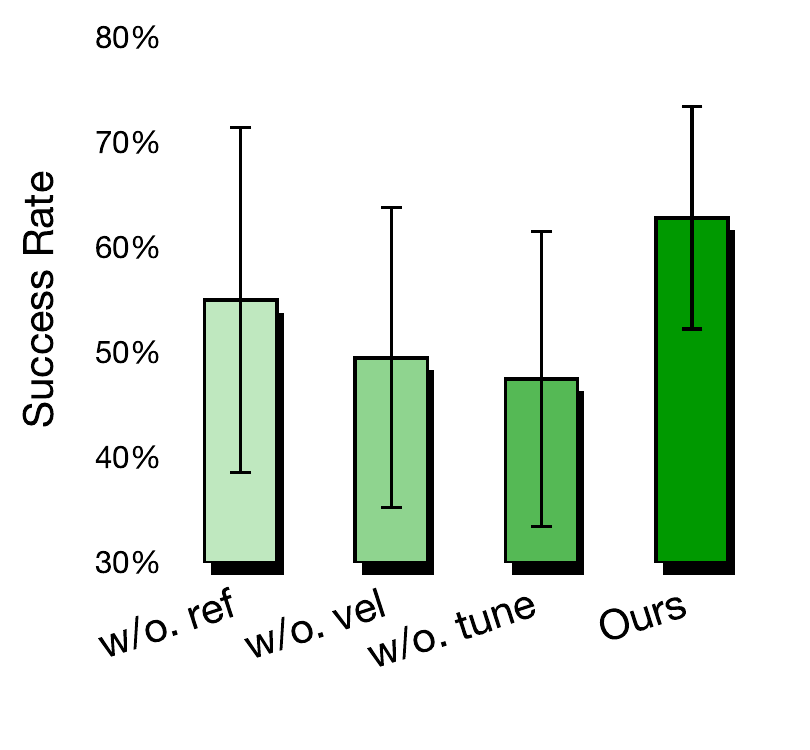}
        \caption{Ablation on proxy training}
        \label{fig:ablation_avg}
    \end{subfigure}
    \hfill
    \begin{subfigure}[b]{0.384\linewidth}
        \centering
        \includegraphics[width=\linewidth]{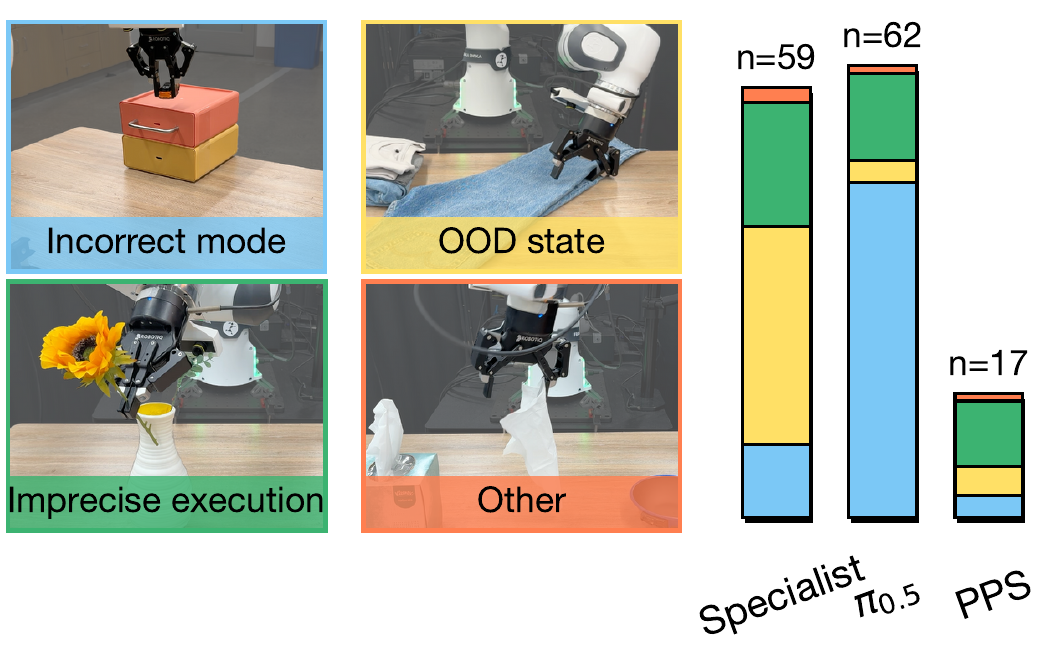}
        \caption{Failure analysis}
        \label{fig:failure}
    \end{subfigure}
    \caption{\textbf{Analysis.} Sensitivity to steering strength, training ablations, and failure-mode breakdown.}
    \label{fig:ablations}
\end{figure*}




\section{Limitations}
\hl{While \MethodAcronym{} offers an effective path to adapting frozen VLAs, the current study is limited in three respects. First, we study the single-task regime; the pipeline extends to multi-task setups by supplying the task identity as input, but how a shared proxy transfers manipulation primitives across tasks, environments, and embodiments remains open. Second, we validate only on flow-matching VLAs, although the principle extends to diffusion policies through the score-to-velocity correspondence and to Gaussian or categorical action heads through density-ratio or logit-space combinations. Third, \MethodAcronym{} relies on the base and the proxies playing complementary roles: the residual elicits task-specific behaviors the frozen base rarely samples on its own, while the base supplies broad manipulation priors, robustness, and recovery. Our results therefore depend on a strong base, and when the base is substantially weaker than the proxy, \MethodAcronym{} may offer less benefit than training a specialist from scratch.}

\section{Conclusions}

\hl{We introduce \MethodName{} (\MethodAcronym), an inference-time adaptation method that specializes a frozen flow-matching VLA from a small demonstration set through a calibrated velocity-space residual between two lightweight proxy policies. We show that under the linear flow-matching path this residual implements a density-ratio correction to the base at each noise level, and we identify the conditions under which it isolates the change induced by task supervision. Across 12 manipulation tasks, \MethodAcronym{} lifts the state-of-the-art $\pi_{0.5}$ base policy by 55\% absolute success rate on average and delivers zero-to-one gains on tasks the base never solves, while retaining recovery behaviors that LoRA fine-tuning degrades. The same recipe transfers across VLA families, lifting $\pi_0$ from 5\% to 55\%, and requires only forward velocity predictions from the base, making the method applicable even to closed-source or API-only base policies. More broadly, our results suggest that proxy-based inference-time adaptation, originally developed for discrete-token language models, admits a principled and calibrated analog for continuous-action generative policies, opening a general path for specializing large generalist policies without accessing or modifying their weights.}


\clearpage


\bibliography{main}  


\clearpage
\appendix

\section{Implementation Details}
\label{sec:exp_details}

This section provides the implementation details of \MethodAcronym{} and the baselines. \Cref{sec:appendix_network_details} describes the network architectures of our proxy and baseline models, \Cref{sec:appendix_control_details} the control and sampling configuration, and \Cref{sec:appendix_training_details} the training hyperparameters shared among all methods.

\subsection{Network Details}
\label{sec:appendix_network_details}
Both proxy policies (reference and task) share the same lightweight architecture: a DINOv3~\cite{simeoni2025dinov3} ViT-S/16 visual encoder followed by a small ($\sim$12M) Gemma transformer action expert. Following the $\pi_0$ convention, the action expert fuses the noised action with a sinusoidal time embedding (concatenation followed by an MLP) and regresses the flow-matching velocity. The visual encoder is trained jointly with the expert (not frozen), giving each proxy about 32M trainable parameters in total, roughly $1\%$ of the frozen base VLA $\pi_{0.5}$~\cite{intelligence2504pi0} ($\sim$3B parameters). The \textbf{Specialist} baseline uses this same architecture trained from scratch. The \textbf{Residual} baseline uses the same architecture and additionally encodes the base policy's predicted action through an extra action encoder, predicting a correction on top of it. \textbf{LoRA}~\cite{hu2022lora} inserts low-rank adapter matrices into both the VLM backbone and the action expert of the base model, leaving the remaining base weights frozen. \textbf{DSRL}~\cite{wagenmaker2025steering} follows the implementation of its original paper, training a policy over the base model's initial noise distribution; we use the offline Noise-Aliased variant.

\subsection{Control and Sampling Setup}
\label{sec:appendix_control_details}
All methods share the control and sampling configuration of the base policy, following the DROID setup~\cite{khazatsky2024droid} and the $\pi$ model configuration~\cite{black2024pi_0,intelligence2504pi0}. Policies run at a control frequency of 15\,Hz. Each forward pass uses 10 Euler denoising steps and predicts an action horizon of 15 steps, of which the first 8 are executed before replanning. \MethodAcronym{} applies the guided velocity in \eqref{eq:guided_velocity} at every one of the 10 denoising steps, and the proxies use the same schedule and action representation as the base, as required by the assumptions in \Cref{app:velocity-derivation}. Unless stated otherwise, we use a fixed steering strength $\gamma=0.4$ for every task, without per-task tuning.

\subsection{Training Hyperparameters}
\label{sec:appendix_training_details}
We train all proxy policies and baselines for 20{,}000 gradient steps with a batch size of 64, by which point every method has converged. Following the $\pi_0$ convention~\citep{black2024pi_0}, the flow-matching noise level $k$ in \eqref{eq:task-loss} is drawn from a shifted $\mathrm{Beta}(1.5,1)$ distribution rather than uniformly over $[0,1]$. The reference distillation loss in \eqref{eq:ref-loss} does not sample $k$: it is supervised at the noise levels the base sampler actually visits along its own denoising trajectory. The exception is \textbf{DSRL}, whose reinforcement-learning objective on a small dataset is unstable for longer training: we train it for 1{,}000 steps and use the last checkpoint before its RL loss diverges.

\section{Additional Details of Experimental Setup}

This section describes the tasks and evaluation methodology used throughout our experiments. \Cref{sec:appendix_task_details} lists the language instruction, success criterion, and demonstration collection for every real-world and simulation task, and \Cref{sec:appendix_evaluation} details the protocol used to measure success rates.

\subsection{Task Design}
\label{sec:appendix_task_details}
\Cref{tab:task_details} lists the language instruction and success criterion for each of the 8 real-world tasks and 4 simulation tasks used in our evaluation. The instruction column is the natural-language prompt given to the VLA at the start of every rollout. The success-criterion column describes the success condition.

For real-world tasks, our hardware setup follows the DROID dataset~\cite{khazatsky2024droid} with one third-person camera and one wrist camera.
We collect 50 demonstrations per task by teleoperation.
For simulated tasks, we use assets from ArtVIP~\cite{jin2025artvip} and the Synthesis asset pack to build our simulation environments. The 50 demonstrations are generated with MimicGen~\cite{mandlekar2023mimicgen} from 5 teleoperated source demonstrations.

\begin{table*}[t]
    \centering
    \caption{\textbf{Task design.} Language instruction and success criterion for each task. The top block lists the 8 real-world tasks; the bottom block lists the 4 simulation tasks.}
    \label{tab:task_details}
    \small
    \renewcommand{\arraystretch}{1.4}
    \begin{tabular}{@{}l p{0.40\linewidth} p{0.40\linewidth}@{}}
        \toprule
        \textbf{Task} & \textbf{Prompt} & \textbf{Success criterion} \\
        \midrule
        Tissue  & ``draw a tissue from the tissue box and wipe the pot'' & The tissue contacts the inside of the pot in a wiping motion. \\
        Utensil & ``put the utensils in the utensil holder'' & The utensils end up placed inside the holder. \\
        Lever   & ``push down the toaster lever'' & The lever is pushed all the way to the bottom. \\
        Shoes   & ``remove the lid of the box and take the shoes out of the box'' & The shoes are placed outside the box. \\
        Coffee  & ``open the coffee maker lid, put the pod inside the coffee maker, close the lid'' & The coffee pod is inside the coffee maker with the lid closed. \\
        Jeans   & ``fold the jeans and put them on the cloth pile'' & The folded jeans rest on top of the cloth pile. \\
        Flower  & ``insert the flower into the vase'' & The stem of the flower is inside the vase. \\
        Drawer  & ``open the drawer, put the toy cars in the drawer, close the drawer'' & The toy cars are inside the drawer and the drawer is closed. \\
        \midrule
        Fruit   & ``put the pear and apple on the electronic scale'' & Both the pear and the apple rest on the scale. \\
        Egg     & ``open the pot by the lid and put the egg in it'' & The egg is inside the pot after the lid has been removed. \\
        Tea     & ``pour the tea from the teapot into the cup'' & The teapot spout is tilted over the cup at a pouring angle. \\
        Capsule & ``open the coffee maker lid and put the pod in it'' & The pod is inside the coffee maker. \\
        \bottomrule
    \end{tabular}
\end{table*}

\subsection{Evaluation Protocol}
\label{sec:appendix_evaluation}
\paragraph{Real-world.} For each real-world task, we randomize the initial layout of the objects while keeping the robot's initial pose fixed, and run 10 rollouts per method per task.

\paragraph{Simulation.} For each simulation task, we run 100 rollouts per method per task using the same set of seeds across methods, so that every method is evaluated from identical initial configurations.
\section{Additional Results}
\label{sec:additional_results}

This section presents additional results that complement the main paper. We break the aggregate numbers down into per-task success rates (\Cref{sec:appendix_per_task}), analyze how \MethodAcronym{} scales with trainable parameters and demonstrations (\Cref{sec:appendix_scale}), and report training, inference, and parameter cost (\Cref{sec:appendix_efficiency}). We then show that \MethodAcronym{} transfers to other base policies (\Cref{sec:appendix_pi0}) and compare against DSRL on a standard benchmark (\Cref{sec:appendix_libero}). Finally, we study cross-modality steering (\Cref{sec:multimodal}), \MethodAcronym{}'s robustness to mid-rollout perturbations (\Cref{sec:appendix_perturbation}), the alignment of the learned velocity fields (\Cref{sec:appendix_cosine}), and the steering behavior itself (\Cref{sec:appendix_steering_vis}).

\subsection{Per-Task Success Rates}
\label{sec:appendix_per_task}
\Cref{fig:per_task_real,fig:per_task_sim} break down the averaged numbers reported in the main paper into per-task success rates. Each bar group is a single task; the rightmost group repeats the macro-average from the main figure for reference. The per-task view shows that \MethodAcronym{} can deliver zero-to-one improvements on the out-of-distribution real-world tasks (Coffee, Jeans, Flower, Drawer), where DSRL fails to gain any improvement. In simulation, \MethodAcronym{} similarly leads on every task.

\begin{figure*}[t]
    \centering
    \includegraphics[width=\linewidth]{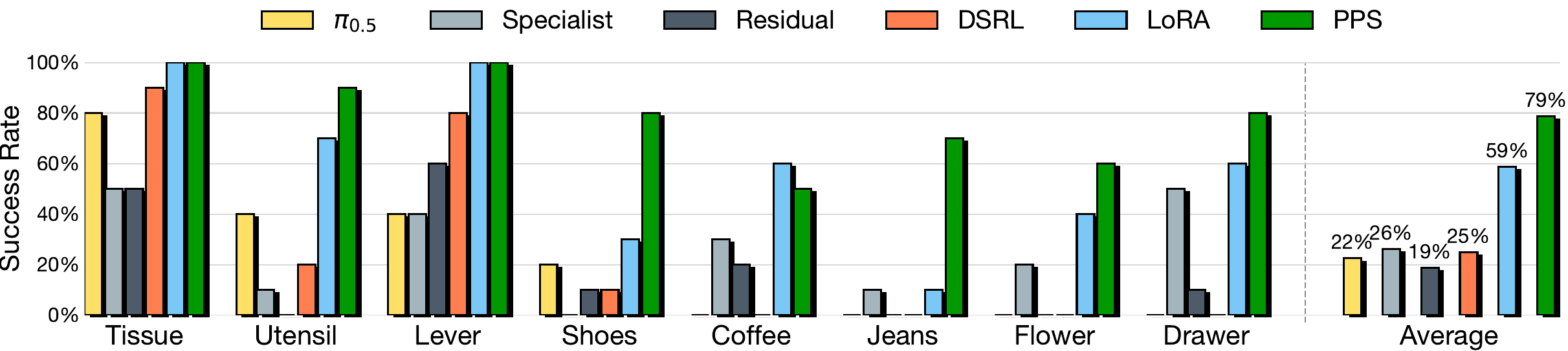}
    \vspace{-4mm}
    \caption{\textbf{Per-task real-world results.} Success rate of each method on the 8 real-world tasks.}
    \label{fig:per_task_real}
\end{figure*}

\begin{figure*}[t]
    \centering
    \includegraphics[width=\linewidth]{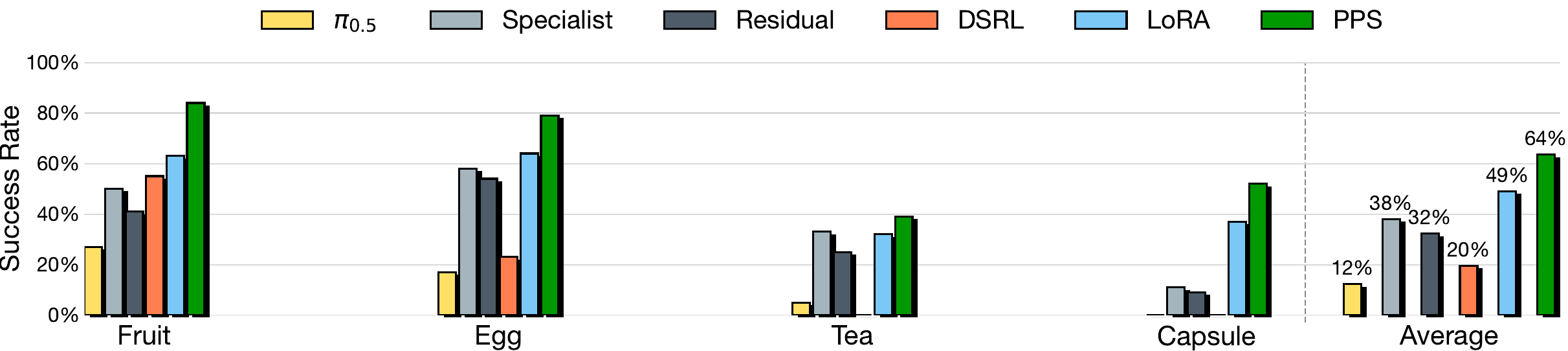}
    \vspace{-4mm}
    \caption{\textbf{Per-task simulation results.} Success rate of each method on the 4 simulation tasks.}
    \label{fig:per_task_sim}
\end{figure*}

\begin{figure*}[t]
    \centering
    \begin{subfigure}[b]{0.32\linewidth}
        \centering
        \includegraphics[width=\linewidth]{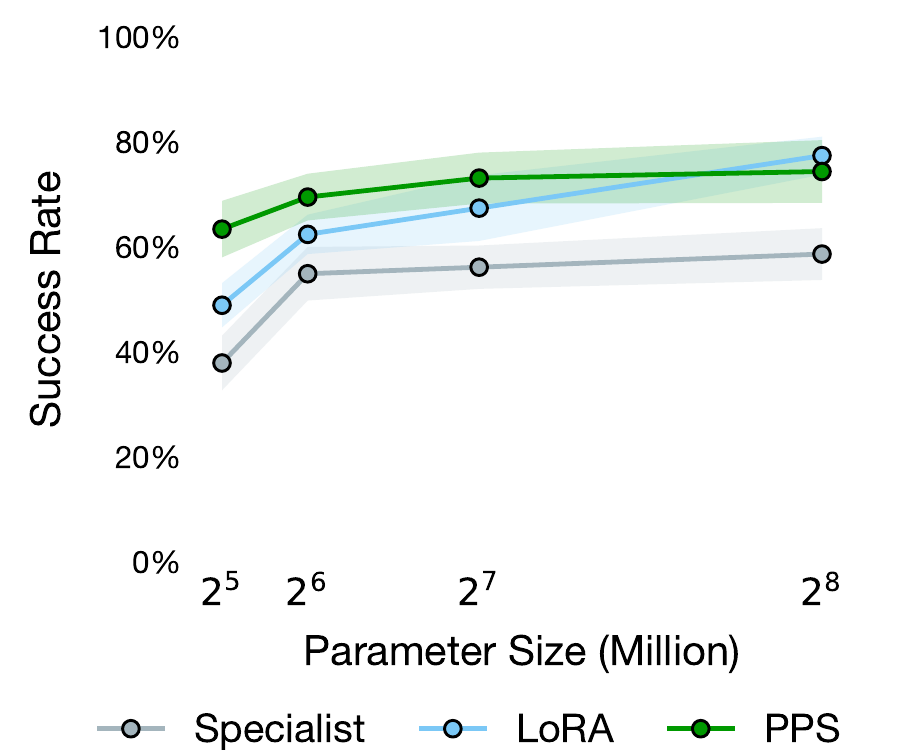}
        \caption{Parameter scaling on four tasks.}
        \label{fig:scale}
    \end{subfigure}
    \hfill
    \begin{subfigure}[b]{0.32\linewidth}
        \centering
        \includegraphics[width=\linewidth]{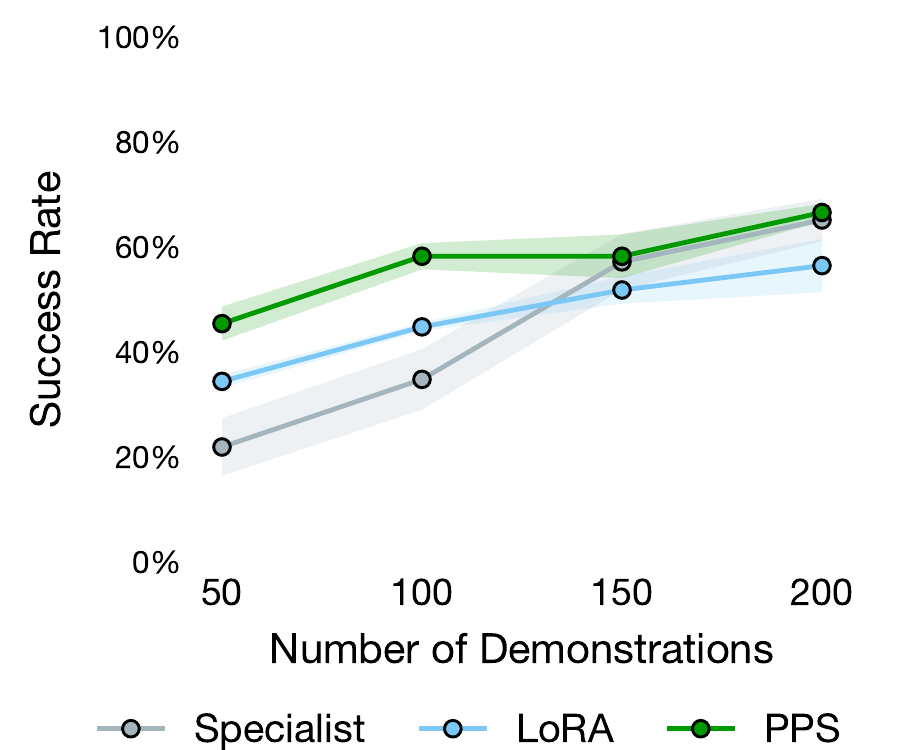}
        \caption{Data scaling on Tea and Capsule.}
        \label{fig:data_scale}
    \end{subfigure}
    \hfill
    \begin{subfigure}[b]{0.32\linewidth}
        \centering
        \includegraphics[width=\linewidth]{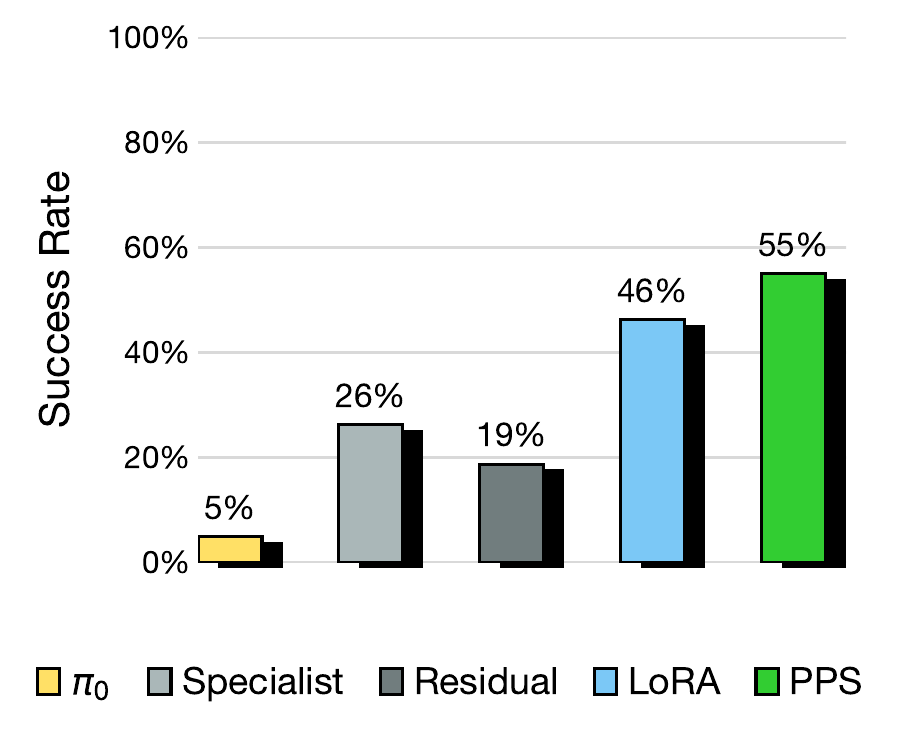}
        \caption{$\pi_0$ real-world results.}
        \label{fig:pi0_avg}
    \end{subfigure}
    \caption{\textbf{Additional results.}
    (a) Average simulated success rate of \MethodAcronym{} against \textbf{Specialist} and \textbf{LoRA} as the trainable-parameter budget grows from 32M to 256M.
    (b) Average simulated success rate on the two lowest-success tasks (Tea and Capsule) as the number of demonstrations grows from 50 to 200.
    (c) Average real-world success rate across the 8 real-world tasks on the $\pi_0$ family. \MethodAcronym{} lifts $\pi_0$ from 5\% to 55\%, indicating that the steering recipe transfers across flow-matching VLAs.}
    \label{fig:additional_exp}
\end{figure*}

\subsection{Scaling up Parameters and Data}
\label{sec:appendix_scale}
\Cref{fig:scale} compares \MethodAcronym{}, \textbf{LoRA}, and \textbf{Specialist} across trainable-parameter budgets from $2^5$M to $2^8$M on the 4 simulation tasks. \MethodAcronym{} is remarkably parameter-efficient: it reaches 64\% success with just 32M trainable parameters, already within 10 points of the best performance either method achieves at any budget, while LoRA requires $8\times$ more parameters to reach the same performance. \textbf{Specialist}, which does not access the base prior, plateaus around 54\% even at 256M parameters, confirming that access to the pretrained VLA's prior rather than simply model capacity drives the performance advantage.

\Cref{fig:data_scale} complements this with a data-scaling sweep on the two lowest-success simulation tasks (Tea and Capsule), growing the number of demonstrations from 50 to 200. \MethodAcronym{} leads at every data budget, with its largest margin in the low-data regime; the from-scratch \textbf{Specialist} improves steadily with more data and catches up near the 200-demonstration budget.

\subsection{Training, Inference, and Parameter Efficiency}
\label{sec:appendix_efficiency}
Because \MethodAcronym{} never backpropagates through the frozen base model, its adaptation cost is substantially lower than weight fine-tuning. Under an identical training configuration on a single A6000 GPU, the complete distill-then-fine-tune pipeline (both proxies) takes \emph{2.9 hours} against \emph{16.3 hours} for \textbf{LoRA}, roughly a $5\times$ reduction.

At inference, \MethodAcronym{} evaluates two additional 32M proxies per denoising step. Since the proxies do not depend on the base model's intermediate activations, they run in parallel with the base policy on a single 4090 GPU, so the added cost is small: \emph{6\,ms} of extra latency per action chunk and \emph{0.13\,GB} of additional GPU memory.

In terms of trainable parameters, the scaling study in \Cref{sec:appendix_scale} shows that LoRA requires roughly $8\times$ more trainable parameters to match the success rate \MethodAcronym{} reaches at 32M.

\subsection{Other Base Policies}
\label{sec:appendix_pi0}
We further apply the \MethodAcronym{} recipe to $\pi_0$~\cite{black2024pi_0} to test whether the steering procedure depends on the specific base policy. \Cref{fig:pi0_avg} reports the average success rate across the 8 real-world tasks. Unsteered $\pi_0$ succeeds on 5\% of trials on average, largely due to 0\% scores on six of the eight tasks (Lever, Shoes, Coffee, Jeans, Flower, Drawer). \MethodAcronym{} lifts the average to 55\%, including zero-to-1 improvements on those six tasks. This confirms that the velocity-space residual targets a property of flow-matching action sampling rather than being $\pi_{0.5}$-specific, and that \MethodAcronym{} remains compatible with other base policies in the same family.

\subsection{Comparison on a Standard Benchmark}
\label{sec:appendix_libero}
Our main evaluation deliberately targets tasks the base VLA rarely solves, since that is the regime where adaptation matters. For completeness, we also compare \MethodAcronym{} against \textbf{DSRL}~\cite{wagenmaker2025steering} on the LIBERO task used in the DSRL paper, under the same base policy $\pi_0$. Both methods lift $\pi_0$ from $17\%$ to $100\%$ success. Performance saturates on this benchmark, so it does not discriminate between the two approaches; the per-task results in \Cref{sec:appendix_per_task} show that the methods separate sharply once the base policy no longer succeeds on its own.

\subsection{Cross-modality Steering}
\label{sec:multimodal}
\MethodAcronym{} also enables injecting new sensing modalities into base models since the standalone proxy policies can accept different inputs than the base model. \Cref{fig:multimodal} shows the two simulated tasks \textbf{Marble}, which takes a point cloud as input to locate the target objects whose visual texture blends in the background, and \textbf{Phone}, which uses sound spectrograms to distinguish the ringing phone from two visually identical phones. \MethodAcronym{} outperforms all baselines, including \textbf{Specialist} and \textbf{DSRL}, which also take the additional modality, and \textbf{LoRA-MM}, which wires the new modality directly into the VLA backbone through end-to-end fine-tuning. 

\begin{figure*}[t]
    \centering
    \includegraphics[width=\linewidth]{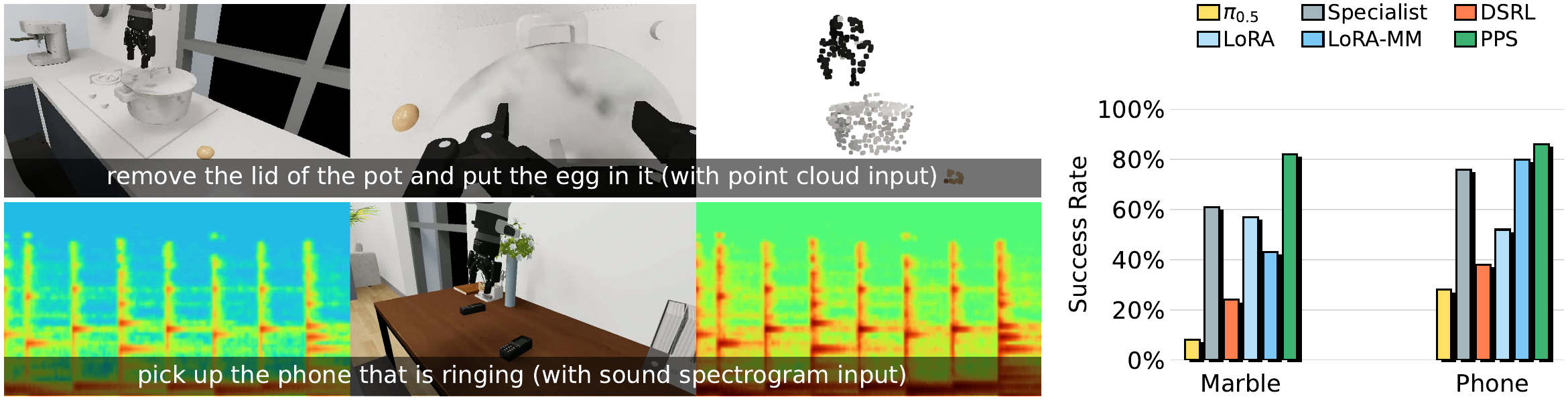}
    \vspace{-4mm}
    \caption{\textbf{Cross-modality steering.} \MethodAcronym{} injects new modalities into the RGB-only base policy.}
    \label{fig:multimodal}
\end{figure*}

\subsection{Robustness Under Mid-Rollout Perturbations}
\label{sec:appendix_perturbation}
\Cref{fig:perturbation} compares \MethodAcronym{} and \textbf{LoRA} on the Jeans task under human interventions during execution. After the human unfolds the partially folded jeans, LoRA cannot recover, and the policy overfits to demonstration-like motions that no longer suit the state. 
\MethodAcronym{} instead adapts to the interventions, repairs the fold, and even survives a second ``drag'' perturbation later in the rollout, eventually completing the task. This supports the claim in \Cref{sec:comparative} that leaving the base VLA frozen preserves the recovery and out-of-distribution behavior the pretrained policy already exhibits, while LoRA's weight-level adaptation overwrites them.

\Cref{fig:disturbance} quantifies this effect. We report $\Delta$SR, the change in success rate when mid-rollout disturbances are applied relative to the undisturbed evaluation, so a value closer to zero indicates a policy that better retains its ability to recover. \textbf{LoRA} loses $15.8$ points and degrades roughly twice as much as \MethodAcronym{}, which loses only $7.8$. Removing the reference proxy is more damaging still: \textbf{\MethodAcronym{} w/o ref} loses $26.1$ points. This is consistent with the role of the reference proxy described in \Cref{sec:appendix_cosine}: without it, the residual carries proxy-specific error in addition to the task-induced change, and the extra noise erodes precisely the base priors that recovery depends on.

\begin{figure*}[t]
    \centering
    \includegraphics[width=\linewidth]{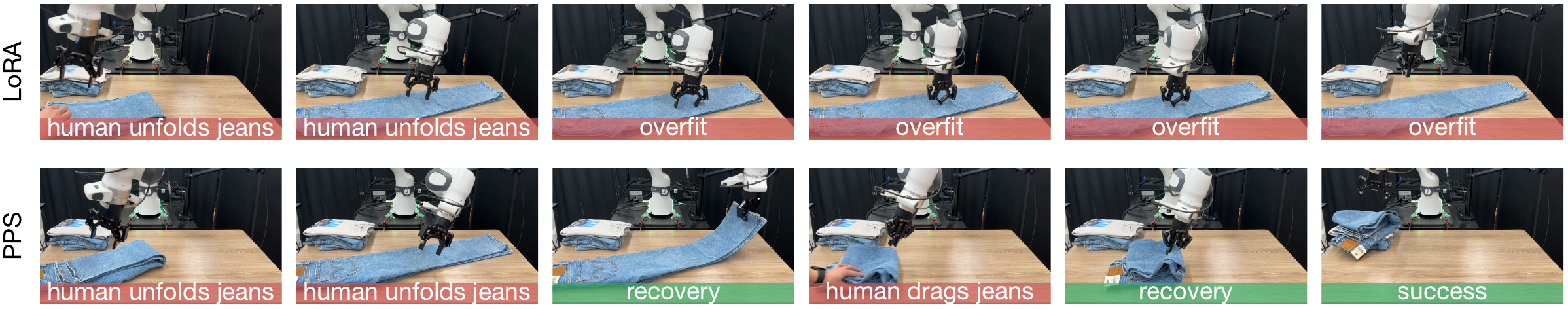}
    \vspace{-4mm}
    \caption{\textbf{Robustness under mid-rollout perturbations.} \textbf{Top (LoRA):} after the first human ``unfold'' intervention LoRA fails to recover and overfits to the demonstration distribution. \textbf{Bottom (\MethodAcronym):} the same rollout under two consecutive interventions (``unfold'', ``drag'') --- \MethodAcronym{} recovers from each and still completes the task on the last frame. Red banners mark perturbation events; green banners mark recovery and the final success.}
    \label{fig:perturbation}
\end{figure*}

\begin{figure}[t]
    \centering
    \begin{subfigure}[t]{0.336\linewidth}
        \centering
        \includegraphics[width=\linewidth]{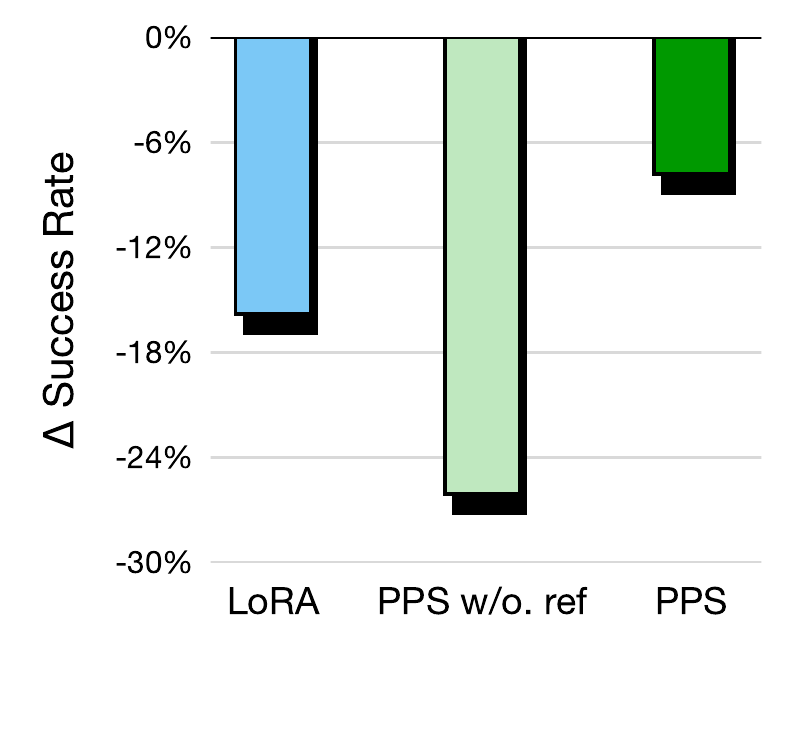}
        \caption{Disturbance robustness}
        \label{fig:disturbance}
    \end{subfigure}
    \hfill
    \begin{subfigure}[t]{0.336\linewidth}
        \centering
        \includegraphics[width=\linewidth]{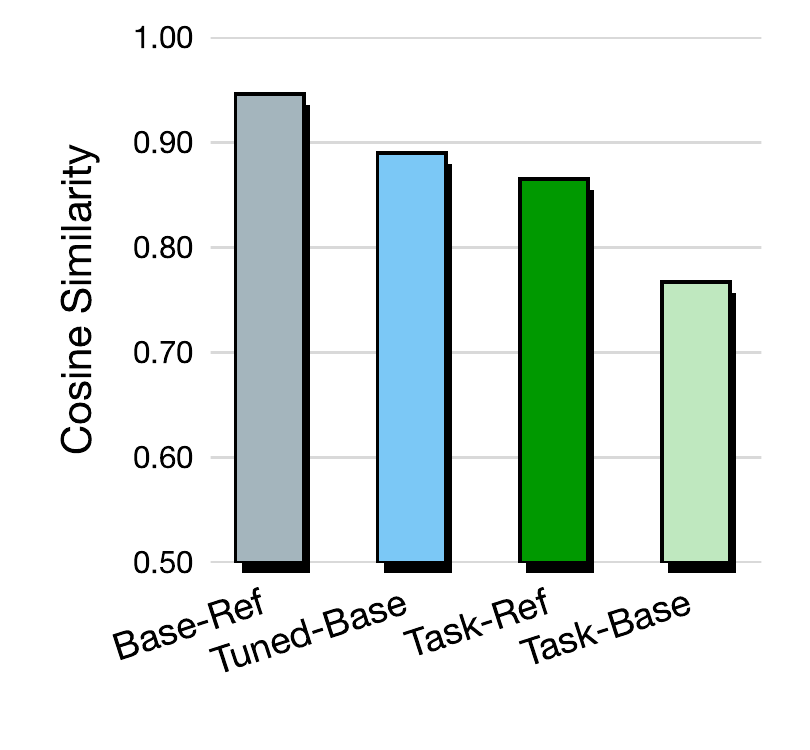}
        \caption{Velocity alignment}
        \label{fig:cosine}
    \end{subfigure}
    \caption{\textbf{Evidence that \MethodAcronym{} preserves the base priors.}
    (a) Change in success rate ($\Delta$SR) when disturbances are applied mid-rollout, relative to the undisturbed evaluation; bars closer to zero are better.
    (b) Average cosine similarity between velocity fields along the steered denoising trajectory. Tuned denotes the base policy after fine-tuning on the task demonstrations, whose shift relative to the base calibrates the size of a genuine task-induced change.}
    \label{fig:preservation}
\end{figure}

\begin{figure*}[t]
    \centering
    \includegraphics[width=\linewidth]{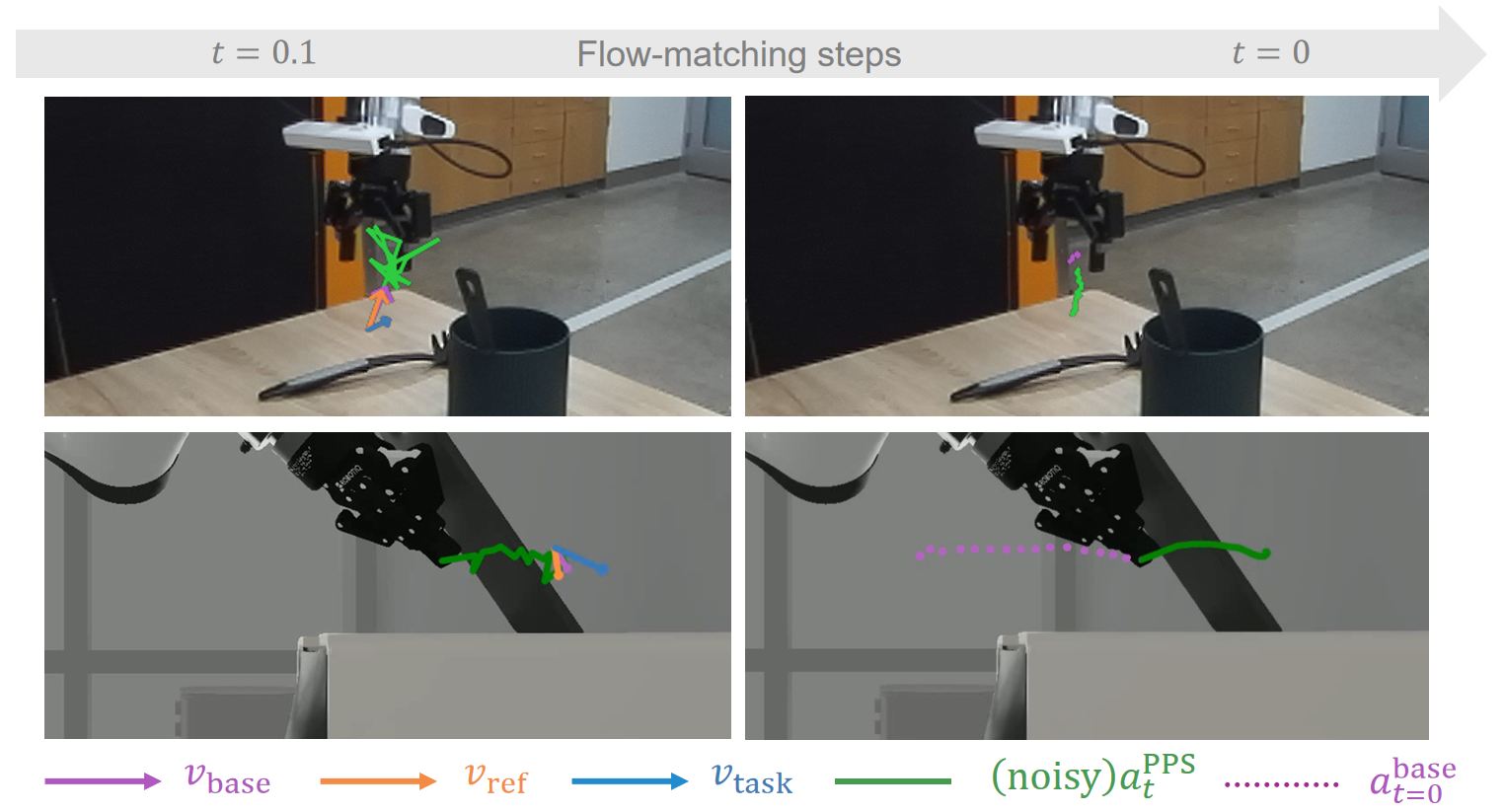}
    \caption{\textbf{Steering visualization.}     
    Velocity vectors and sampled actions during flow-matching denoising. At each denoising step, all policies take the noisy action $(\mathrm{noisy})~a_t^{\mathrm{PPS}}$ as input. The base velocity $v_{\mathrm{base}}$ (purple) points toward an action mode that does not solve the task, while the task proxy velocity $v_{\mathrm{task}}$ (blue) points toward the success mode and the reference proxy velocity $v_{\mathrm{ref}}$ (orange) tracks the base behavior. $a_t^{\mathrm{PPS}}$ (green) is updated using the guided velocity $v_{\mathrm{base}}+\gamma(v_{\mathrm{task}}-v_{\mathrm{ref}})$ throughout the denoising trajectory and converges to the task-success region. The dotted purple action $a_{t=0}^{\mathrm{base}}$ shown in the final column denotes the action obtained by denoising the same initial sample using only the base model, which converges to the incorrect mode.
    }
    \label{fig:steer_vis}
\end{figure*}

\subsection{Velocity Alignment Analysis}
\label{sec:appendix_cosine}
The ablation in the main paper shows that replacing $v_{\mathrm{ref}}$ with the frozen base velocity degrades performance. Here we examine why directly. \Cref{fig:cosine} reports the average cosine similarity between pairs of velocity fields, measured along the steered denoising trajectory on the simulation tasks.

Two observations support the design. First, the reference proxy is an accurate stand-in for the base: $v_{\mathrm{ref}}$ and $v_{\mathrm{base}}$ agree to $0.946$ on the steered path (and $0.935$ on the base model's own path), confirming that on-policy distillation aligns the reference with the base in the region where guidance is applied.

Second, and more importantly, the comparison calibrates how large a genuine task-induced shift should be. Fine-tuning the base itself yields a Tuned--Base similarity of $0.890$: this is the magnitude of the behavioral change that task supervision actually induces. The proxy pair diverges by a comparable amount (Task--Ref, $0.865$). Task--Base, however, falls considerably further, to $0.767$. The additional drop is not task signal. It reflects the architecture and capacity gap between the 32M proxy and the 3B base policy, which the reference proxy shares and therefore cancels, but the base does not. Using $v_{\mathrm{task}}-v_{\mathrm{base}}$ as the residual thus mixes the task-induced change with proxy-specific approximation error, which is what the \textbf{w/o ref} ablation measures and what the disturbance results in \Cref{sec:appendix_perturbation} show to be costly.

\subsection{Steering Visualization}
\label{sec:appendix_steering_vis}
\Cref{fig:steer_vis} visualizes how the velocity residual reshapes the denoising trajectory step by step. We pick two inference snapshots from the \textbf{Utensil} and \textbf{Capsule} tasks and overlay the velocity vector predictions of the base, reference proxy, and task proxy on the last action horizon, together with the shared noisy action chunks. $v_{\mathrm{task}}$ separates from $v_{\mathrm{base}}$ and points toward the task-success mode, while $v_{\mathrm{ref}}$ mimics the base-like vector, which points toward an incorrect action mode.
The task-reference residual $v_{\mathrm{task}}-v_{\mathrm{ref}}$ adds this difference to the base sampler at every flow-matching step, bending the steered trajectory toward the success mode. The final action $a_{t=0}^{\mathrm{PPS}}$ differs visibly from the would-be base action $a_{t=0}^{\mathrm{base}}$.

\section{Velocity Derivation}
\label{app:velocity-derivation}


\MethodAcronym{} defines the steered policy as a product-of-experts correction to the frozen base,
\begin{equation}\label{eq:vg-clean-target}
\pi_{\mathrm{PPS}}(a\mid o,l)\;\propto\;\pi_{\mathrm{base}}(a\mid o,l)\,\left[\pi_{\mathrm{task}}(a\mid o)\big/\pi_{\mathrm{ref}}(a\mid o)\right]^{\gamma}.
\end{equation}
Flow-matching policies do not expose normalized densities. They return only velocity predictions along the linear interpolation path \(x_k=k\epsilon+(1-k)a\), with \(\epsilon\sim\mathcal{N}(0,I)\) and \(k=1\) noise, \(k=0\) clean. This section shows that the correction in \eqref{eq:vg-clean-target}, when applied at each noise level, has an \emph{exact} velocity-space form.

\paragraph{Assumptions.}
\emph{(A1) Shared schedule.} The base policy and both proxies use the same action representation and the same linear interpolation path, so the score-to-velocity relation derived below carries identical coefficients for all three models. Without a shared schedule the three velocity fields are not expressed in a common frame and their difference has no distributional meaning.

\emph{(A2) Common support and integrability.} At every noise level, $\pi_{\mathrm{task}}$ and $\pi_{\mathrm{ref}}$ share support with $\pi_{\mathrm{base}}$, and the corrected density below has a finite normalizer, so the density ratio and its logarithm are well defined.

\emph{(A3) Per-level correction.} We \emph{define} the target at noise level $k$ by applying the same product-of-experts correction to the noised marginals. Writing \(p_{j,k}\) for the marginal density of \(x_k\) under policy \(\pi_j\), \(j\in\{\mathrm{base},\mathrm{task},\mathrm{ref}\}\),
\begin{equation}\label{eq:vg-level-target}
p_{\mathrm{PPS},k}(x)\;\propto\;p_{\mathrm{base},k}(x)\,\left[p_{\mathrm{task},k}(x)\big/p_{\mathrm{ref},k}(x)\right]^{\gamma}.
\end{equation}

Assumptions (A1) and (A2) are design choices that \MethodAcronym{} satisfies by construction: the proxies inherit the base model's flow-matching schedule and action representation, and the reference proxy is distilled from the base on the same observations. Assumption (A3) is different in kind, and we state its status plainly. It is a \emph{modeling choice}, not a consequence of \eqref{eq:vg-clean-target}. Noising does not commute with taking products and ratios of densities, so the family \(\{p_{\mathrm{PPS},k}\}_{k\in[0,1]}\) defined by \eqref{eq:vg-level-target} is in general \emph{not} the family of noise-level marginals of the clean target \eqref{eq:vg-clean-target}. The two coincide at \(k=0\), and trivially for \(\gamma=0\), but not in between. Everything below is exact given (A1)--(A3).

\paragraph{Score form of the correction.}
Taking \(\nabla_x\log\) of both sides of \eqref{eq:vg-level-target}, the normalizer is independent of \(x\) and drops out, leaving an additive score correction
\begin{equation}\label{eq:vg-score}
\nabla_x\log p_{\mathrm{PPS},k}=\nabla_x\log p_{\mathrm{base},k}+\gamma\big(\nabla_x\log p_{\mathrm{task},k}-\nabla_x\log p_{\mathrm{ref},k}\big).
\end{equation}

\paragraph{From scores to velocities.}
It remains to turn this score identity into one between velocities. For the linear path, the marginal velocity \(v(x,k)=\mathbb{E}[\epsilon-a\mid x_k=x]\) of any policy is an affine function of the noised variable and its marginal score \(s_k=\nabla_x\log p_k\). Indeed, with \(\hat{a}=\mathbb{E}[a\mid x_k=x]\) and \(\hat{\epsilon}=\mathbb{E}[\epsilon\mid x_k=x]\), linearity gives \(x=(1-k)\hat{a}+k\hat{\epsilon}\) and \(v=\hat{\epsilon}-\hat{a}\), while \(x_k\mid a\sim\mathcal{N}((1-k)a,k^2 I)\) yields \(\hat{\epsilon}=-k\,s_k\); eliminating \(\hat{a}\) and \(\hat{\epsilon}\) gives
\begin{equation}\label{eq:vg-affine}
v(x,k)=A(k)\,x+B(k)\,s_k(x),\qquad A(k)=-\tfrac{1}{1-k},\quad B(k)=-\tfrac{k}{1-k}.
\end{equation}
The coefficients depend only on the schedule, so by (A1) the same relation holds for \(v_{\mathrm{base}}\), \(v_{\mathrm{task}}\), \(v_{\mathrm{ref}}\), and the steered field \(v_{\mathrm{PPS}}\) with their respective scores. Note that \eqref{eq:vg-affine} is stated for \(k\in[0,1)\): the coefficients diverge at the pure-noise endpoint \(k=1\), and at \(k=0\) we have \(B(0)=0\), so the relation cannot be inverted for the score there.

\begin{proposition}\label{prop:vg}
Assume \emph{(A1)--(A3)}. For every \(k\in(0,1)\), the marginal velocity field of the level-\(k\) target \eqref{eq:vg-level-target} satisfies
\begin{equation}\label{eq:vg-guided-velocity}
v_{\mathrm{PPS}}(x,k,o,l)=v_{\mathrm{base}}(x,k,o,l)+\gamma\big[v_{\mathrm{task}}(x,k,o)-v_{\mathrm{ref}}(x,k,o)\big].
\end{equation}
\end{proposition}

\begin{proof}
Fix \(k\in(0,1)\), so \(B(k)\neq0\). By \eqref{eq:vg-affine}, each policy satisfies \(B(k)\,s_{j,k}=v_j-A(k)\,x\). Multiplying the score identity \eqref{eq:vg-score} through by \(B(k)\) and substituting this relation for each of the four fields gives
\[
v_{\mathrm{PPS}}-A(k)x=\big(v_{\mathrm{base}}-A(k)x\big)+\gamma\Big[\big(v_{\mathrm{task}}-A(k)x\big)-\big(v_{\mathrm{ref}}-A(k)x\big)\Big].
\]
The \(A(k)x\) terms cancel inside the task--reference difference, and the remaining pair cancels across the equation, which yields \eqref{eq:vg-guided-velocity}. At \(k=0\) the identity holds trivially, since \(B(0)=0\) makes \(v_j(x,0)=-x\) for every policy and both sides reduce to \(-x\).
\end{proof}

\Cref{eq:vg-guided-velocity} is the velocity-space form of the steering rule used at inference. It inherits a natural consistency property from the distribution-level correction: when the task and reference proxies agree, the residual vanishes and $v_{\mathrm{PPS}}=v_{\mathrm{base}}$, so the base policy is left undisturbed. The scalar $\gamma$ controls how strongly the base velocity is tilted toward the task proxy.

\end{document}